%% file: arxiv.tex
\documentclass[twoside,11pt]{article}

\usepackage{jmlr2e}

\usepackage{amsmath}
\usepackage{xcolor}
\usepackage[ruled,linesnumbered]{algorithm2e}
\usepackage{ifthen}
\usepackage{comment}
\usepackage{csquotes}
\usepackage{bm}

\newtheorem{fact}{Fact}

\ShortHeadings{ICML 2020 Workshop on Real World Experiment Design and Active Learning}{Workshop on Real World Experiment Design and Active Learning}
\firstpageno{1}

\begin{document}

\title{A Smoothed Analysis of Online Lasso for the Sparse Linear Contextual Bandit Problem}

\author{Zhiyuan Liu \email zhiyuan.liu@colorado.edu\\ 
       \addr Department of Computer Science, University of Colorado, Boulder
       \AND
       \name Huazheng Wang \email hw7ww@virginia.edu \\
       \addr Department of Computer Science, University of Virginia
       \AND
       \name Bo Waggoner \email bo.waggnor@colorado.edu\\
       \addr Department of Computer Science, University of Colorado, Boulder
       \AND
       \name Youjian (Eugene) Liu \email youjian.liu@colorado.edu \\
       \addr Department of Electrical, Computer and Energy Engineering, University of Colorado, Boulder
       \AND
      \name Lijun Chen \email lijun.chen@colorado.edu \\
      \addr Department of Computer Science, University of Colorado, Boulder
      }

\editor{}

\maketitle

\begin{abstract}
We investigate the sparse linear contextual bandit problem where the parameter
$\theta$ is sparse. To relieve the sampling inefficiency, we utilize the ``perturbed adversary" where the context is generated adversarilly but with small random non-adaptive perturbations. We prove that the simple online Lasso supports sparse linear contextual bandit with regret bound $\mathcal{O}(\sqrt{kT\log d})$ even when $d \gg T$ where $k$ and $d$ are the number of effective and ambient dimension, respectively. Compared to the recent work from \cite{sivakumar2020structured}, our analysis does not rely on the precondition processing, adaptive perturbation (the adaptive perturbation violates the i.i.d perturbation setting) or truncation on the error set. Moreover, the special structures in our results explicitly characterize how the perturbation affects exploration length, guide the design of perturbation together with the fundamental performance limit of perturbation method. Numerical experiments are provided to complement the theoretical analysis. 

\end{abstract}
\section{Introduction}

Contextual bandit algorithms have become a referenced solution for sequential decision-making problems such as online recommendations \citep{li2010contextual}, clinical trials \citep{durand2018contextual}, dialogue systems \citep{upadhyay2019bandit} and  anomaly detection \citep{ding2019interactive}. It adaptively learns the personalized mapping between the observed contextual features and unknown parameters such as user preferences, and addresses the trade-off between exploration and exploitation \citep{auer2002using,li2010contextual,abbasi2011improved,agrawal2013thompson,abeille2017linear}.

We consider the sparse linear contextual bandit problem where the context is high-dimensional with sparse unknown parameter $\theta$ \citep{abbasi2012online,hastie2015statistical,dash1997feature}, i.e., most entries in $\theta$ are zero and thus only a few dimensions of the
context feature are relevant to the reward. Due to insufficient data samples, the learning algorithm has to be sampling efficiency to support the sequential decision-making. However, the data from bandit model usually does not satisfy the requirements for sparse recovery such as Null Space condition \citep{cohen2009compressed}, Restricted isometry property (RIP) \citep{donoho2006compressed}, Restricted eigenvalue (RE) condition \citep{bickel2009simultaneous}, Compatibility condition \citep{van2009conditions} and so on. To achieve the desired performance, current works has to consider the restricted problem settings, e.g., the unit-ball, hypercube or i.i.d. arm set \citep{carpentier2012bandit,lattimore2015linear,kim2019doubly,bastani2020online}, the parameter with Gaussian prior\citep{gilton2017sparse}. One exception is the Online-to-Confidence-Set Conversions \citep{abbasi2012online} which considers the general setting but suffers from computation inefficiency.

In this paper, we tackle the sparse linear bandit problem using \emph{smoothed analysis} technique \citep{spielman2004smoothed,kannan2018smoothed}, which enjoys efficient implementation and mild assumptions. Specifically, we consider the perturbed adversary setting where the context is generated adversarially but perturbed by small random noise. This setting interpolates between an i.i.d. distributional assumption on the input, and 
the worst-case 
of 
fully adversarial contexts. Our results show that with a high probability, 
the 
perturbed adversary inherently guarantees 
the 
(linearly) strong convex condition for 
the 
low dimensional case and 
the 
restricted eigenvalue (RE) condition for 
the 
high dimensional case, which is a key property required by 
the 
standard Lasso regression. We 
prove that the simple online Lasso supports sparse linear contextual bandits with regret bound $\mathcal{O}(\sqrt{kT\log d})$. We also provide numerical experiments to complement the theoretical analysis.

	We also notice the recent work from \cite{sivakumar2020structured} using smoothed analysis for structured linear contextual bandits. Compared to their work, our proposed method has the following advantages: (1) Our analysis only relies on the simple online Lasso instead of precondition processing and truncation on the error set. Although preconditioning transfers the non-zero singular value to 1, this could amplify the noise, and the preconditioned noises are no longer i.i.d., which makes concentration analysis difficult and the estimation unstable \citep{jia2015preconditioning}. We also observe this effect in the numeric experiments. 
	(2) Their proof relies on the assumption that perturbations 
	that 
	need to be adaptively generated based on 
	the 
	observed 
	history of 
	the chosen contexts. Instead, our analysis is based on 
	the 
	milder assumption that the perturbation is i.i.d. and non-adaptive. (3) Their regret does not describe the full picture of the effect of variance of the perturbation. Our analysis explicitly show how the perturbation affects the exploration length, guide the design of perturbation together with the fundamental performance limit of perturbation method.   

\section{Model and Methodology}
	In the bandit problem, at each round $t$, the learner pulls an arm $a_t$ among $m$ arms (we denote the arm sets by $[m]$\footnote{In this paper, we denote by $[n]$ the set $[1,\cdots,n]$ for positive integer $n$.}, that is, $a_t \in [m]$) and receives the corresponding noisy reward $r_{a_t}^t$. The performance of the learner is evaluated by the regret $\mathcal{R}$ which quantifies the total loss because of not choosing the best arm $a_t^*$ during $T$ rounds:
	\begin{align}
	    \mathcal{R}(T) = \sum_{t=1}^{T}(r_{a^*_t}^t - r_{a_t}^t).
	\end{align}
	In this paper, we focus on the sparse linear contextual bandit problem. Specially, each arm $i$ at round $t$ is associated with a feature (context) vector $\mu_i^t \in \mathbb{R}^d$. The reward of that arm is assumed to be generated by the noisy linear model which is the inner product of arm feature and an unknown $S$-sparse parameter $\theta^{*}$ where $S$ denotes the set of effective (non-zero) entries and $|S| = k$. That is, 
	\begin{align}
	    r_{i}^t = \langle \mu_{i}^t, \theta^{*} \rangle + \eta^t, |\theta^{*}|_{0} = k,
	\end{align}
	where $\eta^t$ follows Gaussian distribution $\mathcal{N}(0,\sigma^2)$. 
	To handle the non-convex $L_0$ norm, Lasso regression is the natural way to learn the sparse $\theta^{*}$ with the relaxation from $L_0$ to $L_1$ norm. To achieve the desired performance, the algorithm has to rely on 
	well designed contexts which guarantee sampling efficiency requirements such as Null Space condition \citep{cohen2009compressed}, Restricted isometry property (RIP) \citep{donoho2006compressed}, Restricted eigenvalue (RE) condition \citep{bickel2009simultaneous}, Compatibility condition  \citep{van2009conditions} and so on. However, the data from bandit problems usually does not satisfy these conditions since the contexts could be generated adversarilly. Up to now, deciding on the proper assumptions for sparsity bandit problems is still a challenge \citep{lattimore2018bandit}.
	
	Inspired by the
	smoothed analysis for greedy algorithm of linear bandit problem \citep{kannan2018smoothed}, we consider the \emph{Perturbed Adversary} defined below for the sparse linear contextual bandit problem.
	\begin{definition} {\textbf{Perturbed Adversary} \citep{kannan2018smoothed}.} The perturbed adversary acts as the following at round $t$.
	\begin{itemize}
	    \item Given the current context $\mu_1^t, \cdots, \mu_m^t$ which could be chosen adversarially, the perturbation $e_1^t, \cdots, e_m^t$ are drawn independently from certain distribution. Also, each $e_i^t$ is independently (non-adaptively) produced of the context.
	    \item The perturbed adversary outputs the contexts $(x_1^t,\cdots,$
	    $x_m^t) = (\mu_1^t+e_1^t,\cdots,\mu_m^t+e_m^t)$ as the arm features to the learner.
	\end{itemize}
	\end{definition}
	Let $X \in \mathbb{R}^{d \times t}$ 
	be the context matrix where each column contains one context vector and $Y$ the column vector that contains the corresponding rewards. Based on perturbed adversary setting, we analyze the online Lasso in Algorithm \ref{algo:1} for sparse linear contextual bandit. 
	\begin{algorithm}[ht]
		\caption{Online Lasso For Sparse Linear Contextual Bandit Under Perturbed Adversary}\label{algo:1}
		Initialize $\theta^{0}$, $X$ and $Y$. \\
		\For{$t = 1,2,3,\cdots,T$}{
			The perturbed adversary produces $m$ context $[x_{1}^t,...,x_{m}^{t}]$.\\
			The learner greedily chooses the arm $i = \arg \max_{j \in [m]} \langle x_j^t,\theta^t\rangle$, observes the reward $r_i^t$, appends the new observation $(x_{i}^t,r_{i}^t)$ to $(X,Y)$, and 
			updates $\theta^{t+1}$ by the Lasso regression:
	    \begin{align} \label{eq:opt}
		\theta^{t+1} = \arg \min_{\theta}~~G(\theta; \lambda^t):= \|Y - X^{\top} \theta\|_2^2 + \lambda^t\|\theta\|_1.
		\end{align}
		}
	\end{algorithm}
	Generally speaking, our analysis considers two cases using different techniques, one for the low dimensional case when $d < T$ and the other for the
	high dimensional case when $d \gg T$. For the low dimensional case, the analysis utilizes random matrix theory \citep{tropp2012user} to prove that with a high probability, the minimum eigenvalue of scaled sample covariance matrix is increasing linearly with round $t$; for the
	high dimensional case, the 
	RE condition is guaranteed with the help of Gaussian perturbation's property \citep{raskutti2010restricted} that the nullspace of context matrix under Gaussian perturbation cannot contain any vectors that are ``overly'' sparse when $t$ is larger than some threshold. The properties of both cases support $\mathcal{O}(\sqrt{\frac{k\log d}{t}})$ parameter recovery of Lasso regression under noisy environment which leads to $\mathcal{O}(\sqrt{kT\log d})$ regret.
	
	\subsection{Low Dimensional Case}
    We first consider the low dimensional case when $d < T$. 
	Under the perturbed adversary setting, we then define the property named perturbed diversity which is adopted from \cite{bastani2017mostly}.
	\begin{definition}{\textbf{Perturbed Diversity.}} Let $e_i^t \sim D$ on $\mathbb{R}^d$. Given any context vector $\mu_i^t$, we call $x_i^t$ perturbed diversity if for $x_i^t = \mu_i^t + e_i^t$, the minimum eigenvalue of sample covariance matrix under perturbations satisfies
	\begin{align*}
	    \lambda_{\min}\left( \mathbb{E}_{e_i^t \sim D}\left[x_i^t(x_i^t)^{\top}  \right]\right) \geq \lambda_{0},
	\end{align*} where $\lambda_{0}$ is a positive constant.
	\end{definition}
	Intuitively speaking, perturbed diversity guarantees that each context provides at least certain information about all coordinates of $\theta^*$ from the expectation which is helpful to recover the support of the parameter via regularized method. We can find several distributions $D$ that
	could make the perturbed diversity happen, e.g., the 
	Gaussian distribution. However, without any restriction, $x_i^t$ could be very large and out of the realistic domain. Instead, the value of each dimension (we denote by $x_i^t(j)$ the $j$-th dimension of $x_i^t$) should lie in a bounded interval, in the meanwhile, the total energy of context vector is bounded by certain constant, i.e., $\|x_i^t\|_2^2 \leq R^2$. This motivates us to consider the perturbed diversity under censored perturbed adversary.  
	\begin{lemma} \label{lemma:1}
		Given the context vector $\mu_i^t \in \mathbb{R}^d$ and $|\mu_i^t(j)| \leq  q_j$ for each $j \in [d]$, we define the censored perturbed context $x_i^t$ under $e_i^t \sim \mathcal{N}(\bm{0},\sigma_1^2 I)$ as follows:
		\begin{align}
		x_i^t(j) = \begin{cases} \mu_i^t(j) + e_i^t(j), ~~~\text{if} ~~|\mu_i^t(j)+e_i^t(j)| \leq  q_j,\\
		~~~q_j, ~~~~~~~~~~~~~\text{if} ~~~\mu_i^t(j)+e_i^t(j) >  q_j, \\
		~-q_j, ~~~~~~~~~~~\text{if} ~~~\mu_i^t(j)+e_i^t(j) < -q_j. \\
		\end{cases}
		\end{align}
		Then $x_i^t$ has the perturbed diversity with $\lambda_0 = g(\frac{2q}{\sigma_1},0)\sigma^2$, where $q = \min_j q_j$ and $g(\cdot,\cdot)$ is a composite function of the 
		probability density function $\phi(\cdot)$ and the 
		cumulative distribution function $\Phi(\cdot)$ of the
		normal distribution. Please refer to equation \eqref{equ:g} for more details.
	\end{lemma}
The proof is provided in the appendix and one can easily extend 
it to the case where $e_i^t \sim \mathcal{N}(\bm{0},\Sigma)$. Based on Lemma \ref{lemma:1}, we can derive that with a high probability, $\lambda_{\min}(XX^{\top})$ grows at least with a linear rate $t$. 
	\begin{lemma} \label{lemma:eig}
		With the censored perturbed diversity, when $t > \frac{2R^2}{g\left(\frac{2q}{\sigma_1},0\right)\sigma_1^2}\log(dT)$,  
		the following is satisfied with 
		probability $1-\frac{1}{T}$: $\lambda_{\min} (X X^{\top}) \geq  g\left(\frac{2q}{\sigma_1},0\right)(1-\tau)\sigma_1^2t,$
		where $\tau = \sqrt{\frac{2R^2}{g\left(\frac{2q}{\sigma_1},0\right)\sigma_1^2t}\log(dT)}$.
	\end{lemma}
	As one can see from Lemma \ref{lemma:eig}, after certain
	number of (implicit) exploration rounds, i.e., $\frac{2R^2}{g\left(2q/\sigma_1,0\right)\sigma_1^2}\log(dT),$ we will have enough information  to support the $\mathcal{O}(\sqrt{\frac{k\log d}{t}})$ parameter recovery by Lasso regression. 
	The regret analysis together with the
	high dimensional case is deferred to the next section.  
    \subsection{High Dimensional Case}
    Now we turn to the high dimensional case when $d \gg T$. During the learning process, the scaled sample covariance matrix $XX^{\top}$ is always rank deficiency which means $\lambda_{\min}(XX^{\top}) = 0$ and Lemma \ref{lemma:eig} based on random matrix theory can not be applied here any more. We then consider the restricted eigenvalue (RE) condition instead. Here the ``restricted'' means that the 
    error $\Delta^t := \theta^t - \theta^*$ incurred by Lasso regression 
    is 
    restricted 
    to 
    a set with special structure. That is, $ \Delta^t \in \mathcal{C}(S;\alpha)$ where 
    \begin{align*}
        \mathcal{C}(S;\alpha) := \{ \theta \in \mathbb{R}^d | \|\theta_{S^c}\|_1 \leq \alpha \|\theta_S\|_1
        \},
    \end{align*}
    and $\alpha \geq 1$ is determined by the 
    regularized parameter $\lambda^t$. In the following section, we focus on $\mathcal{C}(S;3)$ which could be achieved by setting  $\lambda^t = \Theta(2\sigma R \sqrt{2t\log (2d)})$. 
    
The key is to prove that Null space of $X^{\top}$ 
has no overlapping with $\mathcal{C}(S;3)$. It has been proved that special cases in which contexts are purely sampled from special distributions such as Gaussian and Bernoulli distributions, satisfy this property \citep{zhou2009restricted,raskutti2010restricted,haupt2010toeplitz}. We make a further step to show that nullspace of context matrix under Gaussian perturbation cannot contain any vectors that are ``overly'' sparse when $t$ is larger than some threshold. 

    \begin{theorem} \label{theorem:high}
    Consider
    perturbation $e_i^t \sim \mathcal{N}(\bm{0},\Sigma)$ where
    $\|\Sigma^{1/2}\Delta\|_2 \geq \gamma \|\Delta\|_2$ for $\Delta \in \mathcal{C}(S;3)$. If
    $ t > \max(\underbrace{\frac{4c'' q(\Sigma)}{\gamma^2} k \log d}_{\textbf{d}},~\underbrace{\frac{8196aR^2 \lambda_{\max}(\Sigma) \log T}{\gamma^4}}_{\textbf{e}})$\label{equ:explore},
    then with 
    probability $1 - (\frac{c'}{e^{ct}} +\frac{1}{T^a})$,
    we have 
    $\Delta^{\top} XX^{\top} \Delta \geq h t\|\Delta\|_2^2$, where $c,c',c''$ are universal constants, $q(\Sigma) = \max_i \Sigma_{ii}$ and $h = ( \frac{\gamma^2}{64} - R\|\Delta\|_2^2\sqrt{\frac{2a\lambda_{\max}(\Sigma) \log T}{t}})$. 
    \end{theorem}
    
    Moreover, we can design $\gamma^2 = \lambda_{\min}(\Sigma)$. By Rayleigh quotient, one can obtain $\lambda_{\max}(\Sigma) \geq q(\Sigma) =\max_i \Sigma_{ii} \geq \min_i \Sigma_{ii} \geq \lambda_{\min}(\Sigma) = \gamma^2.$ 
    
    We then discuss how perturbations will affect the exploration length. First, \emph{ 
    the larger perturbation does not indicate the less regret.} Results of \cite{sivakumar2020structured} show 
    that 
    the regret is $\mathcal{O}(\frac{\log T \sqrt{T}}{\sigma_1})$ where $\sigma_1$ is the perturbation's variance, and suggests choosing larger $\sigma_1$ leads to smaller regret bounds. However this is not the full picture that shows the effect of the perturbation's variance. Our results show that increasing the variance of the perturbation has limited effect over the necessary exploration and regret, which reveals theoretical limit of the
     perturbation method. Specifically, in the term \textbf{(d)} of Theorem \ref{theorem:high}, 
     no matter how large the variance is, 
     the 
     term $\frac{q(\Sigma)}{\gamma^2} \geq 1$. So 
     $4c''k\log d$ is the necessary exploration length and cannot be improved. Second, \emph{Condition Number and 
     the 
     SPR (the signal to perturbation ratio)} are important factors. 
     The 
     condition number $\textbf{Cond}(\Sigma)$ controls both term $\textbf{(d)}$ and 
     $\textbf{(e)}$, 
     e.g., $\frac{q(\Sigma)}{\gamma^2} \leq \frac{\lambda_{\max}(\Sigma)}{\lambda_{\min}(\Sigma)} = \textbf{Cond}(\Sigma)$. This also shows that the optimal perturbation design 
     will choose 
     $\Sigma = \sigma_1 I$. In \textbf{(e)} 
     of Theorem \ref{theorem:high}, 
     $\frac{R^2}{\gamma^2}$ can be regarded as the ratio 
     between the 
     energy of 
     the 
     unperturbed context and 
     the 
     perturbation energy. This ratio shows the trade-off between \emph{exploration} and \emph{fidelity}. That is, 
     a 
     large variance 
     not only 
     reduces the exploration 
     (
     meanwhile, the lower bound is guaranteed by \textbf{(e)}) but also reduces the fidelity of original context.   
	\section{Regret Analysis}
	Based on the properties we 
	have 
	proved for the low and high dimensional cases, we can 
	obtain 
	the following recovery guarantee by the techniques from 
	the 
	standard Lasso regression \citep{hastie2015statistical}.
	\begin{lemma} \label{lemma:recovery}
		    If $t > T_e$ and $\lambda^t = 2\sigma R \sqrt{2t\log \frac{2d}{\delta}}$, the Lasso regression under perturbed adversary has the recovery guarantee $\|\theta^t - \theta^{*}\|_2 \leq \frac{3\sigma R}{C}\sqrt{\frac{2k\log 2d/\delta}{t}}$
		with probability $1-\delta$, where $T_e = \frac{2R^2}{g\left(\frac{2q}{\sigma_1},0\right)\sigma_1^2}\log(dT)$, $C =g\left(\frac{2q}{\sigma_1},0\right)(1-\tau)\sigma_1^2$ for 
		the 
		low dimensional case and $T_e = \max(\frac{4c'' q(\Sigma)}{\gamma^2} k \log d ,~ \frac{8196aR^2 \lambda_{\max}(\Sigma) \log T}{\gamma^4})$, $C = \frac{\gamma^2}{64} - R\sqrt{\frac{2a\lambda_{\max}(\Sigma) \log T}{t}}$ for 
		the 
		high dimensional case. \\ 
	\end{lemma}
	We then get the final result in Theorem \ref{theorem:regret} based on all the analysis above. 
	\begin{theorem} \label{theorem:regret}
		The 
		online Lasso for sparse linear contextual bandit under perturbed adversary admits the following regret with 
		probability $1-\delta$.
	\begin{align}
	Regret \leq  2R\left( T_e +\frac{6\sigma R}{C}\sqrt{2kT\frac{\log 2d}{\delta}} \right) = \mathcal{O}(\sqrt{kT\log d}).
	\end{align}
	\end{theorem}

	\section{Conclusion}
	    This paper utilizes the ``perturbed adversary'' where the context is generated adversarially but with small random non-adaptive perturbations to tackle sparse linear contextual bandit problem. We prove that the simple online Lasso supports sparse linear contextual bandit with regret bound $\mathcal{O}(\sqrt{kT\log d})$ for both low and high dimensional cases and show how the perturbation affects the exploration length and the trade-off between exploration and fidelity. Future work will focus on extending our analysis to more challenge setting, i.e., defending against adversarial attack for contextual bandit model.

\appendix
\label{app:theorem}



\vskip 0.2in
\bibliography{reference}
\bibliographystyle{plain}
\section*{Appendix}
\input{appendix}
\end{document}

%% file: appendix.tex
    \begin{lemma}{(A variant of Matrix Chernoff \cite{tropp2012user})} \label{lemma:trop} Consider a finite sequence ${z_t}$ of independent, random, self-adjoint matrices satisfy 
    \begin{align*}
        z_t \succeq 0 ~~\text{and}~~ \lambda_{\max}(z_t) \leq Q~~~\text{almost surely.}
    \end{align*}
    Compute the minimum eigenvalue of the sum of expectations, $\psi_{\min} := \lambda_{\min}(\sum_{t} \mathbb{E}(z_t)).$ Then for $\delta \in [0,1],$ we have
    \begin{align} \label{eq:chernoff}
        \mathbb{P}\left\{\lambda_{\min}(\sum_{t} z_t) \leq (1-\delta)\psi_{\min}\right\} \leq d\left[
        \frac{e^{-\delta}}{(1-\delta)^{1-\delta}}
        \right]^{\psi_{\min}/Q}.
    \end{align}
    Moreover, for any $\psi \leq \psi_{\min}$, we can get
    \begin{align}
        \mathbb{P}\left\{\lambda_{\min}(\sum_{t} z_t) \leq (1-\delta)\psi\right\} \leq d\left[
        \frac{e^{-\delta}}{(1-\delta)^{1-\delta}}
        \right]^{\psi/Q}.
    \end{align}
    \end{lemma}
    \begin{proof}
        Since $\psi \leq \psi_{\min}$, there exists $\delta_1 \in [0,1]$ such that $\psi = \delta_1 \psi_{\min}.$ We have
        \begin{align*}
            (1-\delta)\psi = (1-\delta)\delta_1\psi_{\min} = (1-(\underbrace{1-\delta_1 + \delta\delta_1}_{\delta_2}))\psi_{\min}.
        \end{align*}
        Plugging this into \eqref{eq:chernoff} leads to 
        \begin{align*}
            \mathbb{P}\left\{\lambda_{\min}(\sum_{t} z_t) \leq (1-\delta)\psi\right\} \leq d\left[
        \frac{e^{-\delta_2}}{(1-\delta_2)^{1-\delta_2}}
        \right]^{\psi_{\min}/Q}.
        \end{align*}
        One can easily verify that $\delta_2 \geq \delta$. So 
        \begin{align*}
            \left[
        \frac{e^{-\delta_2}}{(1-\delta_2)^{1-\delta_2}}
        \right]^{\frac{\psi_{\min}}{Q}} \!\!\!\!\leq\!\! \left[
        \frac{e^{-\delta}}{(1-\delta)^{1-\delta}}
        \right]^{\frac{\psi_{\min}}{Q}} \!\!\!\!\leq\!\! \left[
        \frac{e^{-\delta}}{(1-\delta)^{1-\delta}}
        \right]^{\frac{\psi}{Q}}.
        \end{align*}  
        Then we obtain
        \begin{align}
        \mathbb{P}\left\{\lambda_{\min}(\sum_{t} z_t) \leq (1-\delta)\psi\right\} \leq d\left[
        \frac{e^{-\delta}}{(1-\delta)^{1-\delta}}
        \right]^{\psi/Q}.
    \end{align}
    Since $\frac{e^{-\delta}}{(1-\delta)^{1-\delta}} \leq e^{-\delta^2/2}$, so we have the following when $\delta \in [0,1]$: 
    \begin{align}
        \mathbb{P}\left\{\lambda_{\min}(\sum_{t} z_t) \leq (1-\delta)\psi\right\} \leq d\left[
        e^{-\delta^2/2}
        \right]^{\psi/Q}. \label{eq:trop}
    \end{align}
    \end{proof}
    
    \begin{fact} \label{fact:1}
    Let $\eta = [\eta_1, \cdots, \eta_t]^{\top}$ where each $\eta_i$ i.i.d. from $\mathcal{N}(0,\sigma^2).$ Let $X \in \mathbb{R}^{d \times t}$ where each $|X_{ij}| \leq R.$ Then with a high probability $1-\delta$, we have
    \begin{align*}
       \|X \eta\|_{\infty} \leq \sigma R \sqrt{2t\log \frac{2d}{\delta}}.
    \end{align*}
    \end{fact}
    
    \begin{fact}{(Chernoff Bound for Sum of Sub-Gaussian random variables)}\label{fact:2}
    Let $X_1,\cdots,X_n$ be n independent random variables such that $X_i \sim \textbf{subG}(\sigma^2)$. Then for any $a \in \mathbb{R}^{n}$ and $c >=0$, we have
        \begin{align}
            \Pr\left(\sum_{i=1}^{n}a_i X_i < - c\right) \leq \exp{\left(-\frac{c^2}{2\sigma^2\|a\|_2^2}\right)}. 
        \end{align}
    That is, with a high probability at least $1-\delta,$ we have
    \begin{align}
    \sum_{i=1}^{n}a_i X_i > - \sqrt{2\sigma^2\|a\|_2^2 \log \frac{1}{\delta}}.
    \end{align}
    \end{fact} 
    
    \begin{lemma}{(Restricted Eigenvalue Property (Corollary 1 of \cite{raskutti2010restricted}))} Suppose that $\Sigma$ satisfies the RE condition of order $k$ with parameters $(1, \gamma)$ and denote $q(\Sigma) = \max_i \Sigma_{ii}$. Then for universal positive constants $c,c',c''$, if the sample size satisfies 
    \begin{align}
    t > \frac{4c'' q(\Sigma)}{\gamma^2} k \log d,
    \end{align}
    then the matrix $\frac{\Phi \Phi^{\top}}{t}$ satisfies the RE condition with parameters $(1,\frac{\gamma}{8})$ with probability at least $1-\frac{c'}{e^{ct}}$ where $\Phi \in \mathbb{R}^{d \times t}$ and each column is i.i.d. $\mathcal{N}(\bm{0},\Sigma)$. 
    \end{lemma} \label{lemma:noise}
    \subsection*{Proof of Lemma \ref{lemma:1}}
    	\begin{proof}   
		Since $e_i^t(j)$ is independent of each other, we can analyze it by coordinates. To simplify the analysis, we slightly abuse the notations and remove subscript $i$ and superscript $t$ (only within this proof), that is, $x(j) := x_i^t(j)$ and $e(j) := e_i^t(j)$. 
		\begin{align*}
		    \lambda_{\min} \left(\mathbb{E}\left[xx^{\top}\right]\right) &= \min_{\|w\| = 1} w^{\top}\mathbb{E}[xx^{\top}]w\\
		    &= \min_{\|w\|=1} \mathbb{E}(w^{\top}xx^{\top}w) \\
		    &= \min_{\|w\|=1} \mathbb{E}(\langle w,x\rangle)^2) \\
		    & \geq \min_{\|w\|=1} \mathrm{Var}(\langle w,x\rangle) \\
		    & \geq \min_{\|w\|=1} \mathrm{Var}(\langle w,e\rangle) \\
		    & = \min_{\|w\|=1} \sum_{i=1}^{d} (w(i))^2 \mathrm{Var}(e(i)| \textit{censored}~in~[-q_i,q_i])\\
		    & \geq \min_{\|w\|=1} g(2q/\sigma,0)\sigma_1^2\sum_{i=1}^{d} (w(i))^2 \\
		    & = g(2q/\sigma_1,0)\sigma_1^2,
		\end{align*} 
		where $g(2q/\sigma_1,0)$ is according to Lemma \ref{lemma:6}.
	\end{proof}
	\begin{lemma} \label{lemma:6}
	Let $e \sim \mathcal{N}(0,\sigma_1^2)$. For any interval $[a,b]$ which contains $0$ and fixed length $2q$, e.g., $b-a = 2q$, and $q \geq \sigma_1$, we have the following result:
	\begin{align}
	   \mathrm{Var}(e|\textit{censored}~in~[a,b]) \geq g(2q/\sigma_1,0)\sigma_1^2. 
	\end{align}
	\end{lemma}
	\begin{proof}
	    We first derive the variance for two sided censored Gaussian Distribution. Denote $\alpha = a/\sigma_1$ and $\beta = b/\sigma_1$. For the truncated Gaussian distribution, we have
	    \begin{align*}
	        \mathbb{E}(e|e \in [a,b]) &= \sigma_1 \frac{\phi(\alpha) - \phi(\beta)}{\Phi(\alpha) - \Phi(\beta)} = \sigma_1\rho.\\
	        \mathrm{Var}(e|e \in [a,b]) &= \sigma_1^2(1+\underbrace{\frac{\alpha\phi(\alpha) - \beta\phi(\beta)}{\Phi(\alpha) - \Phi(\beta)} - \rho^2}_{\Lambda}).
	    \end{align*}
	    Then we calculate the variance of two sided censored Gaussian distribution by
	    \begin{align*}
	    \mathrm{Var}(e|\textit{censored}~in~[a,b]) = \mathbb{E}_{y}[ \mathrm{Var}(e|y)] + \mathrm{Var}_{y}[ \mathbb{E}(e|y)],   
	    \end{align*}
	    where $y$ denotes the event $e \in [a,b]$. After some basic calculations, we can get the following result:
	    \begin{align}
	       \mathrm{Var}(e|\textit{censored}~in~[a,b]) &= \sigma_1^2(\Phi(\beta)- \Phi(\alpha))(1+\Lambda) \nonumber\\
	       &~~~~~+ \sigma_1^2[(\rho - \beta)^2(\Phi(\beta)- \Phi(\alpha))(1-\Phi(\beta)+ \Phi(\alpha)) \nonumber\\
	       &~~~~~+2(\beta - \alpha)(\rho-\beta)(\Phi(\beta)- \Phi(\alpha))\Phi(\alpha) \nonumber\\
	       &~~~~~+ (\beta - \alpha)^2(1-\Phi(\alpha))\Phi(\alpha)] \nonumber\\
	       &= g(\beta,\alpha)\sigma_1^2 \label{equ:g} .
	    \end{align}
	    One can show that (1) $\mathrm{Var}(e|\textit{censored}~in~[a,b])$ achieves minimum  when $a=0$ or $b =0$ by the first order optimality condition. (2) $\mathrm{Var}(e|\textit{censored}~in~[0,b])$ is an increasing function w.r.t $b$. Based on (1) and (2), we obtain  
	    \begin{align*}
	        \mathrm{Var}(e|\textit{censored}~in~[a,b]) \geq \mathrm{Var}(e|\textit{censored} \in [0,2q]) = g(2q/\sigma_1,0)\sigma_1^2, 
	    \end{align*}
	\end{proof}
	\subsection*{Proof of Lemma \ref{lemma:eig}}
		\begin{proof}
	    At round $t$, we have
	    \begin{align*}
	      &~~~~\lambda_{\min}(\mathbb{E}(XX^{\top}))\\ 
	      &= \lambda_{\min}(\mathbb{E}(\sum_{i=1}^{t}(x_{a_i}^i (x_{a_i}^{i})^{\top})) = \lambda_{\min}(\sum_{i=1}^{t} \mathbb{E}(x_{a_i}^i (x_{a_i}^{i})^{\top})) \\
	      &\geq \sum_{i=1}^{t} \lambda_{\min}( \mathbb{E}(x_{a_i}^i (x_{a_i}^{i})^{\top})),
	    \end{align*}
	    where the second equality is due to the independence of each round's perturbation and the inequality comes from the fact that minimum eigenvalue is an super-additive operator.
	    
	    For the censored Gaussian perturbation, $\lambda_{\min}( \mathbb{E}(x_{a_i}^{i}(x_{a_i}^{i})^{\top})) \geq g(\frac{2q}{\sigma_1},0))\sigma_1^2$ based on Lemma \ref{lemma:1}. So $\lambda_{\min}(\mathbb{E}(XX^{\top})) \geq g(\frac{2q}{\sigma_1},0)\sigma_1^2 t$.
	    
	    Based on \eqref{eq:trop} of Lemma \ref{lemma:trop} and $\lambda_{\max}(x^i_{a_i}(x^i_{a_i})^{\top}) \leq \|x^i_{a_i}\|_2^2 \leq R^2$, one can obtain
	    \begin{align*}
	          \mathbb{P}\left\{\lambda_{\min}(X X^{\top}) \leq g( \frac{2q}{\sigma_1},0)(1-\tau)\sigma_1^2t \right\} \leq d\left[
        e^{-\tau^2/2}
        \right]^{\frac{g(2R/\sigma_1,0)\sigma_1^2t}{R^2}}.   
	    \end{align*}
	    Let $\frac{1}{T} = d\left[
        e^{-\tau^2/2}
        \right]^{\frac{g(2R/\sigma_1,0)\sigma_1^2t}{R^2}}$ and one can get the final result.
	\end{proof}
	\subsection*{Proof of Theorem 1}
	  \begin{proof}
    To simplify the analysis, we slightly abuse the notation and denote the unperturbed context matrix by $\mu$ where each column $\mu_i$ is one context vector. Similarly, denote $e$ to be the perturbation matrix and $e_i$ to be the column vector. We first decompose the $\Delta^{\top} XX^{\top} \Delta$ as follows:
    	\begin{align}
	   \Delta^{\top} XX^{\top} \Delta = \underbrace{\Delta^{\top} \mu\mu^{\top} \Delta}_{\textbf{(a)}}  + 2\underbrace{\Delta^{\top}e\mu^{\top}\Delta}_{\textbf{(b)}} + \underbrace{\Delta^{\top}ee^{\top}\Delta}_{\textbf{(c)}}.  \label{eq:total} 
	\end{align}
	For the term $\textbf{(a)}$ in equation \eqref{eq:total}, one can only show $\textbf{(a)}$ since $\Delta$ could lie in $\textbf{Null}(\mu^{\top})$. For term \textbf{(b)} and \textbf{(c)}, we find both terms high probability lower bounds respectively.  
	
	Now consider a positive definite matrix $\Sigma$ and we can design that $\Sigma$ such that it satisfies the RE, that is, $\|\Sigma^{1/2}\Delta\|_2 \geq \gamma \|\Delta\|_2$. Based on Lemma \ref{lemma:noise}, we can derive the following for term \textbf{(c)}. For universal positive constants $c,c',c''$, if the sample size satisfies
    \begin{align}
    t > \frac{4c'' q(\Sigma)}{\gamma^2} k \log d,
    \end{align}
    where $q(\Sigma) = \max_{i} \Sigma_{ii},$
    then with probability at least $1-\frac{c'}{e^{ct}}$
    \begin{align}
    \Delta^{\top}ee^{\top}\Delta \geq \frac{\gamma^2}{64}t\|\Delta\|_2^2.\label{equ:b}
    \end{align}
    We then derive a high probability bound for $\textbf{(b)}$. First, we decompose \textbf{(b)} into a weighted sum of i.i.d. Gaussian variable. That is,
    \begin{align}
        \Delta^{\top} e\mu^{\top} \Delta = \sum_{i=1}^{t}(\mu_i^{\top}\Delta)(\Delta^{\top }e_i),
    \end{align}
    where $\mu_i^{T}\Delta$ is the weight and each $\Delta^{\top}e_i \sim \mathcal{N}(0,\Delta^{\top}\Sigma\Delta)$. Based on the Chernoff Bound of weighted sum of sub-Gaussian random variables in Fact \ref{fact:2}, we have
    \begin{align}
        \sum_{i=1}^{t}(\mu_i^{\top}\Delta)(\Delta^{\top }e_i) &\geq - \sqrt{2a\Delta^{\top} \Sigma \Delta \sum_{i=1}^{t}(\mu_i^{\top}\Delta)^2 \log t} \\
        &\geq  -\sqrt{2a\lambda_{\max}(\Sigma)\|\Delta\|_2^2 \sum_{i=1}^{t}R^2\|\Delta\|^2_2 \log t}\\
        & = -Rt\|\Delta\|_2^2\sqrt{\frac{2a\lambda_{\max}(\Sigma) \log t}{t}}. \label{equ:a}
    \end{align} 
    with probability at least $1-\frac{1}{t^{a}}$. We can conclude with probability at least $1 - (\frac{c'}{e^{ct}} +\frac{1}{t^a}),$ both inequality \eqref{equ:b} and \eqref{equ:a} hold. If the round $t$ satisfies
    \begin{align}\label{equ:explore}
    t > \max\left(\underbrace{\frac{4c'' q(\Sigma)}{\gamma^2} k \log d}_{\textbf{d}},~\underbrace{\frac{8196aR^2 \lambda_{\max}(\Sigma) \log t}{\gamma^4}}_{\textbf{e}}\right).
    \end{align}, we have $\textbf{(b)} + \textbf{(c)} \geq h t\|\Delta\|_2^2$, where $h = \left(\frac{\gamma^2}{64} - R\sqrt{\frac{2a\lambda_{\max}(\Sigma) \log t}{t}}\right).$
    \end{proof}
	\subsection*{Proof of Lemma \ref{lemma:recovery}}
	\begin{proof}
		Our proof combines the techniques from smoothed analysis and Lasso regression. Since $\theta^t$ minimizes $G(\theta)$, we have $G(\theta^t) \leq G(\theta^*)$. This yields the following inequality
		\begin{align*}
		    \|X^{\top}\Delta^t\|_2^2 \leq \Delta^t X\eta + \lambda^t(\|\theta^*\|_1 - \|\theta^* + \Delta^t\|_1),
		\end{align*}
		where $\eta$ denotes the noise vector. Note that $\|\theta^*\|_1 = \|\theta^*_S\|$. Furthermore, one can verify that
		$\|\theta^*\|_1 - \|\theta^* + \Delta^t\|_1 \leq \|\Delta^t_S\|_1 -\|\Delta^t_{S^c}\|_1$.
		For $\Delta^t X\eta$, applying H\"o lder's inequality yields
		\begin{align*}
		 \Delta^t X\eta \leq \|\Delta^t\|_1\|X\eta\|_{\infty} \leq \sigma R \sqrt{2t\log \frac{2d}{\delta}} \|\Delta^t\|_1 = \frac{\lambda^t}{2}\|\Delta^t\|_1,  
		\end{align*}
		where the second inequality is due to the fact \ref{fact:1}.
		Combine all above and we obtain
		\begin{align}
		    \|X^{\top}\Delta^t\|_2^2 &\leq \frac{\lambda^t}{2}\|\Delta^t\|_1 + \lambda^t(\|\Delta^t_{S}\|_1 - \|\Delta^t_{S^c}\|_1) \label{eq:set}\\
		    &\leq \frac{3}{2}\lambda^t\|\Delta^t\|_1 \leq \frac{3}{2}\lambda^t\sqrt{k}\|\Delta^t\|_2. \label{eq:1}
		\end{align}
		First from inequality \eqref{eq:set}, we can obtain $\Delta^t \in \mathcal{C}(S;3)$. For low dimensional case, we have $\|X^{\top}\Delta^t\|_2^2 \geq \lambda_{\min}(XX^{\top})\|\Delta^t\|_2^2 \geq  Ct\|\Delta^t\|_2^2$ by Lemma \ref{lemma:eig}, where $C= g\left(\frac{2q}{\sigma_1},0\right)(1-\tau)\sigma_1^2.$ For high dimensional case, we apply Theorem $\ref{theorem:high}$ since $\Delta^t \in \mathcal{C}(S;3)$ and get $\|X^{\top}\Delta^t\|_2^2 \geq Ct\|\Delta^t\|_2^2$ where $C = \frac{\gamma^2}{64} - R\sqrt{\frac{2a\lambda_{\max}(\Sigma) \log T}{t}}.$
		Combine these with inequality \eqref{eq:1} and we get the final result
		\begin{align*}
		    \|\Delta^t\|_2 \leq \frac{3\sigma R}{C}\sqrt{\frac{2k\log 2d/\delta}{t}}.
		\end{align*}
	\end{proof}
		\subsection*{Proof of Theorem \ref{theorem:regret}}
	\begin{proof}
		As for the regret in round $t$, we have
		\begin{align*}
		&~~~~\langle x_{i_t^*}^t,\theta^{*}\rangle -\langle x_{i_t}^t,\theta^{*}\rangle \\
		&= \langle x_{i^*_t}^t,\theta^{*}-\theta^t\rangle - \langle x_{i_t}^t,\theta^{*}-\theta^t\rangle + \langle x_{i_t^*}^t,\theta^{t}\rangle-\langle x_{i_t}^t,\theta^t\rangle\\
		&\leq \langle x_{i_t^*}^t,\theta^{*}-\theta^t\rangle - \langle x_{i_t}^t,\theta^{*}-\theta^t\rangle \\
		&\leq \|\langle x_{i_t^*}^t,\theta^{*}-\theta^t\rangle\|_2 + \|\langle x_{i_t}^t,\theta^{*}-\theta^t\rangle\|_2 \\
		&\leq 2R\|\theta^{*}-\theta^t\|_2,
		\end{align*}
		where the first inequality comes from the greedy choice since $i_t = \arg \max_i \langle x_i^t,\theta^t\rangle$ and the last inequality is due to the censored perturbations. Based on the analysis of low and high dimensional cases, we denote the exploration length as $T_{e}$. During the exploration, we can bound the regret by $2RT_e$.
		So we can derive that
		\begin{align*}
		Regret &= \sum_{t=1}^{T_e} \langle x_{i^*}^t,\theta^{*}\rangle -\langle x_{i_t}^t,\theta^{*}\rangle + \sum_{t=T_e+1}^{T} \langle x_{i^*}^t,\theta^{*}\rangle -\langle x_{i_t}^t,\theta^{*}\rangle \\
		&\leq 2RT_e + 2R\sum_{t=T_e+1}^{T}\|\theta^{*}-\theta^t\|_2 \\
	    & \leq 2RT_e + 2R\sum_{t=T_e+1}^{T}\frac{3\sigma R}{C}\sqrt{\frac{2k\log 2d/\delta}{t}}\\
	    & \leq 2R\left( T_e +\frac{6\sigma R}{C}\sqrt{2kT\frac{\log 2d}{\delta}} \right)
		\end{align*}
		\end{proof}
	\subsection*{Numeric Simulations}
	This section shows the result of numeric simulations. We choose the context's dimension $d = 2000$ where effective dimension $k=20$ and 5 arms for each round. Our sparse bandit learning process only contains 150 rounds with each context vector are randomly generated from the uniform distribution $[0,1]$. Each experiment are repeated 10 times to reduce the effect of the other unnecessary factors. Solid line denotes the average performance and the shadow area contains best and worst performance during repeated running. 
		\begin{figure}[ht!]
	\centering
	\includegraphics[width=0.85\columnwidth]{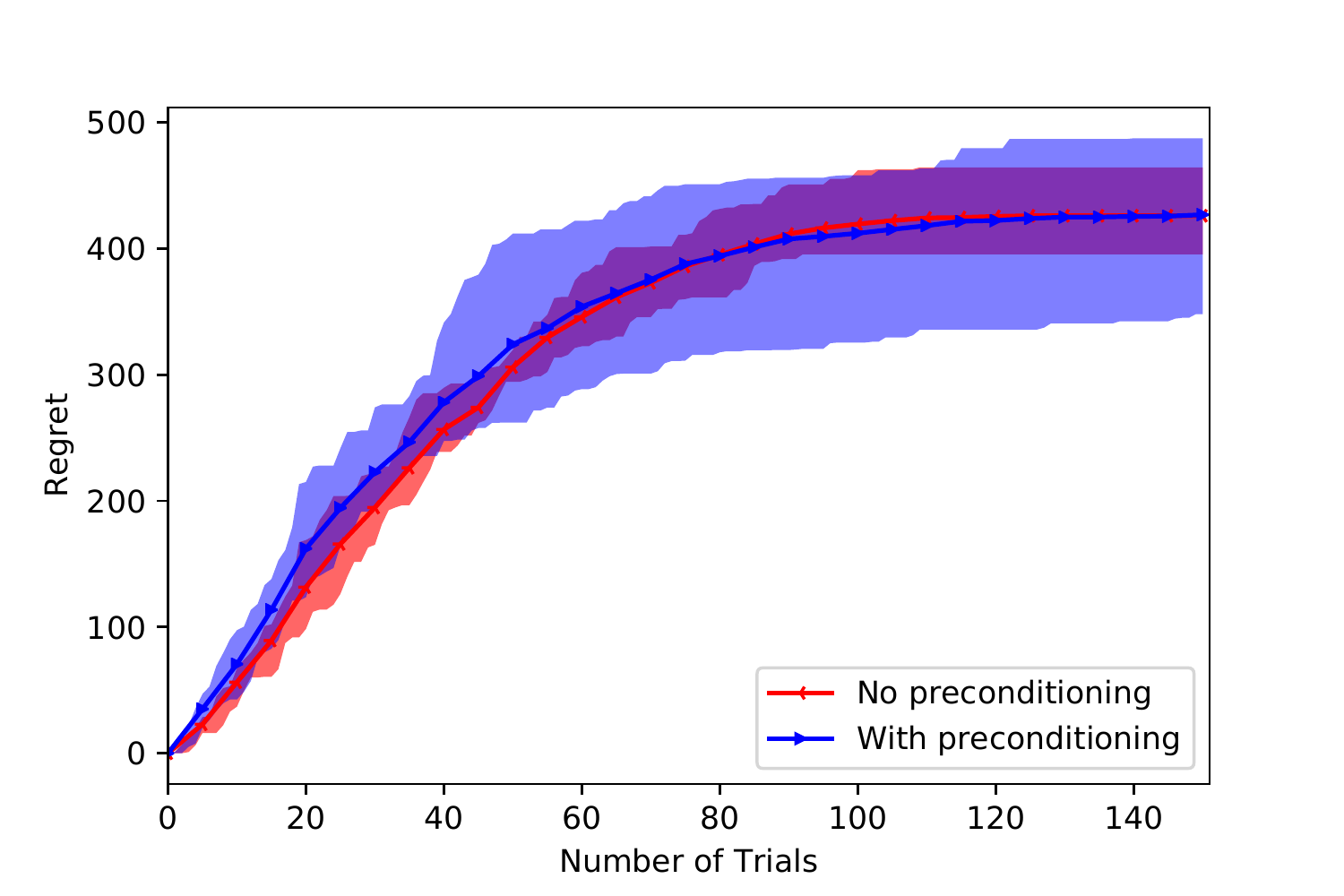}
	\caption{Regret of preconditioning and no preconditioning setting with perturbation variance $0.1$. 
	}\label{fig:1} 
\end{figure}
	We first compare the preconditioning via SVD from the algorithm of  \cite{sivakumar2020structured} which transfers all non-zero singular eigenvalues to 1. Figure \ref{fig:1} shows the regret results with and without preconditioning. In our experiments, we find that the average performance of preconditioning is almost the same as the one without preconditioning (see solid line). Moreover, the performance of preconditioning shows more unstable (see the shadow area). Also, for $d = 2000,$ preconditioning heavily slows down the learning process. The reason could be the noise amplification incurred by preconditioning where \cite{jia2015preconditioning} shows (1) the preconditioned noise are no longer i.i.d. (2) preconditioning can amplify the noise.
	\begin{figure}[ht!]
	\centering
	\includegraphics[width=0.85\columnwidth]{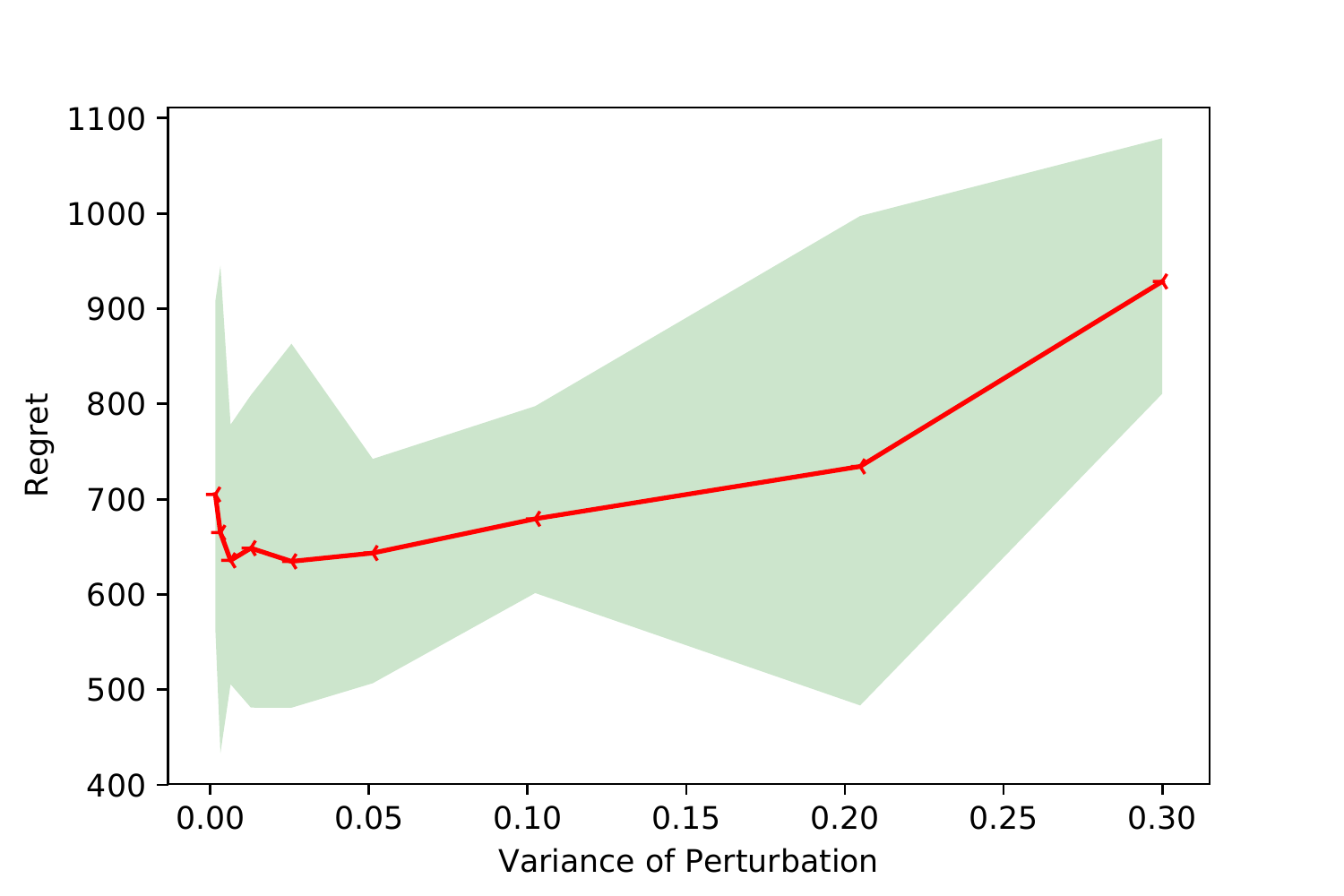}
	\caption{Regret under different perturbation variance.
	}\label{fig:2} 
\end{figure}

We then investigate the performance under different perturbation variance. The result in Figure \ref{fig:2} shows the regret will first decrease then increase which is expected by our analysis. The first decreasing phase is because that the perturbation brings good property to the context matrix. When the perturbation variance becomes large, the context's variance also becomes large which leads to more explorations under the uncertainty environment.  